\pdfoutput=1

\documentclass[english,11pt]{article}
\include{macros}

\begin{document}

\title{A Similarity Measure Between Functions \\ with Applications to Statistical Learning and Optimization}\blfootnote{Author names are sorted alphabetically.}

\author{Chengpiao Huang\thanks{Department of IEOR, Columbia University. Email: \texttt{chengpiao.huang@columbia.edu}.}
	\and Kaizheng Wang\thanks{Department of IEOR and Data Science Institute, Columbia University. Email: \texttt{kaizheng.wang@columbia.edu}.}
}

\date{\today}

\maketitle

\begin{abstract}
In this note, we present a novel measure of similarity between two functions. It quantifies how the sub-optimality gaps of two functions convert to each other, and unifies several existing notions of functional similarity. We show that it has convenient operation rules, and illustrate its use in empirical risk minimization and non-stationary online optimization.
\end{abstract}

\section{Introduction}\label{sec-intro}

Quantifying the closeness between two functions is an essential part of many studies in statistical learning and optimization. For example, in empirical risk minimization, the convergence rate of an empirical minimizer is often derived from studying the concentration of the empirical risk around its population version \citep{BBL03, BBM05, Wai19}. In non-stationary online optimization, the discrepancy between loss functions in different periods reflects the variation of the underlying environment \citep{BGZ15, JRS15, CWW19}. In this note, we present a novel measure of similarity between functions that unifies several existing notions of functional similarity, and illustrate its use in statistical learning and optimization. The measure was first proposed by \cite{HWa23} for studying the problem of online statistical learning under non-stationarity.

\paragraph{Notation.} We use $\RR_+$ to denote the set of non-negative real numbers. For non-negative sequences $\{a_n\}_{n=1}^{\infty}$ and $\{b_n\}_{n=1}^{\infty}$, we write $a_n\lesssim b_n$ or $a_n=\cO(b_n)$ if there exists $C>0$ such that for all $n$, it holds that $a_n \le C b_n$. The diameter of a set $\Omega\subseteq\RR^d$ is defined by $\diam(\Omega) = \sup_{\bx,\by\in\Omega} \| \bx - \by \|_2$. The sup-norm of a function $f:\Omega\to\RR$ is defined by $\|f\|_{\infty} = \sup_{\bx\in\Omega} | f(\bx) |$.

\section{A Measure of Similarity Between Functions}\label{sec-theory}

In this section, we present the measure of functional similarity and its key properties. The presentation in this section largely follows that in Section 5.2 of \cite{HWa23}.

\begin{definition}[Closeness]\label{defn-approx}
Let $\Omega$ be a set. Suppose $f, g : \Omega \to \RR$ are lower bounded and $\varepsilon, \delta \geq 0$. The functions $f$ and $g$ are said to be \textbf{$( \varepsilon , \delta )$-close} if the following inequalities hold for all $\btheta \in \Omega$:
\begin{align*}
& g ( \btheta ) - \inf_{ \btheta' \in \Omega } g(\btheta') \leq e^{\varepsilon}   \bigg(  f ( \btheta ) - \inf_{ \btheta' \in \Omega } f (\btheta') + \delta
\bigg) , \\[4pt]
& f ( \btheta ) - \inf_{ \btheta' \in \Omega } f(\btheta') \leq e^{\varepsilon}   \bigg(  g ( \btheta ) - \inf_{ \btheta' \in \Omega } g (\btheta') + \delta
\bigg) .
\end{align*}
In this case, we also say that $f$ is $(\varepsilon, \delta)$-close to $g$.
\end{definition}

The closeness measure reflects the conversion between the sub-optimality gaps of two functions. Specifically, an approximate minimizer of $f$ is also an approximate minimizer of $g$, up to an additive difference $\delta$ and a multiplicative factor $e^{\varepsilon}$, and vice versa. \Cref{lem_sublevel} below provides a more geometric interpretation through a sandwich-type inclusion of sub-level sets.

\begin{lemma}[Sub-level set characterization]\label{lem_sublevel}
	For any lower bounded $h: \Omega \to \RR$ and $t \in \RR$, define the sub-level set
	\[
	S (h, t) = \bigg\{ \btheta \in \Omega :~ h (\btheta) \leq \inf_{  \btheta' \in \Omega } h (\btheta') + t \bigg\}.
	\]
	Two lower bounded functions $f,g : \Omega \to \RR$ are $(\varepsilon, \delta)$-close if and only if
	\[
	S \big( g, e^{-\varepsilon} t - \delta \big)
	\subseteq  S ( f, t ) \subseteq S \big( g, e^{\varepsilon} (t + \delta) \big) , \qquad \forall t \in \RR.
	\]
\end{lemma}

Intuitively, $\delta$ measures the intrinsic discrepancy between the two functions and $\varepsilon$ provides some leeway. The latter allows for a large difference between the sub-optimality gaps $f ( \btheta ) - \inf_{ \btheta' \in \Omega } f(\btheta') $ and $g ( \btheta ) - \inf_{ \btheta' \in \Omega } g (\btheta') $ when $\btheta$ is highly sub-optimal for $f$ or $g$. 
Thanks to the multiplicative factor $e^{\varepsilon}$, our closeness measure gives a more refined characterization than the supremum metric $\| f - g \|_{\infty} = \sup_{  \btheta \in \Omega } |f(\btheta) - g(\btheta)|$. We illustrate this using the elementary \Cref{eg-sup-norm} below. In \Cref{sec-ERM}, we will see that allowing for such a multiplicative factor $e^{\varepsilon}$ is crucial for obtaining fast statistical rates in empirical risk minimization.

\begin{example}\label{eg-sup-norm}
Let $\Omega = [-1, 1]$ and $a,b \in \Omega$. If $f(\theta) = |\theta - a|$ and $g(\theta) =  2 |\theta - b|$, then $f$ and $g$ are $(\log 2, |a - b| )$-close. In contrast, $\| f - g \|_{\infty} \geq 1$ always, even when $f$ and $g$ have the same minimizer $a=b$. To see this, since $f(-1) = 1 + a$, $g(-1) =2 + 2b$, $f(1) = 1 - a$ and $g(1) = 2 - 2b$, then
\begin{align*}
\|f - g\|_{\infty}
& \ge 
\frac{|f(-1)-g(-1)| + |f(1)-g(1)|}{2}  
=
\frac{|1+2b-a|+|1-(2b-a)|}{2}
\ge 
1.\end{align*}
\end{example}

In \Cref{lem-sufficient}, we provide user-friendly conditions for computing the closeness parameters. In \Cref{sec-applications}, we will instantiate these conditions for statistical learning and optimization problems.

\begin{lemma}\label{lem-sufficient}
Let $\Omega\subseteq \RR^d$ be closed and convex, with $\diam(\Omega)=M<\infty$. Let $f, g: \Omega \to \RR$.
	\begin{enumerate}
		\item\label{lem-sufficient-sup} If $D_0 = \sup_{ \btheta \in \Omega } | f ( \btheta ) - g ( \btheta ) - c |< \infty$ for some $c \in \RR$,	then $f$ and $g$ are $( 0, 2 D_0  )$-close.
		
		\item\label{lem-sufficient-grad-sup} If $D_1 = \sup_{ \btheta \in \Omega } \| \nabla f ( \btheta ) - \nabla g ( \btheta ) \|_2 < \infty$,	then $f$ and $g$ are $( 0, 2 M D_1   )$-close. 
		
		\item\label{lem-sufficient-grad-square} If the assumption in Part \ref{lem-sufficient-grad-sup} holds and there exists $\rho > 0$ such that $g$ is $\rho$-strongly convex over $\Omega$, then $f$ and $g$ are $\big(\log 2,~\frac{2}{\rho} \min \{ D_1^2 ,  \rho M D_1   \} \big)$-close.
		
		\item\label{lem-sufficient-minimizers} Suppose there exist $0 < \rho \leq L < \infty$ such that 
		$f$ and $g$ are $\rho$-strongly convex and $L$-smooth over $\Omega$. In addition, suppose that $f$ and $g$ attain their minima at $\btheta^*_f,\,\btheta^*_g\in\Omega$, respectively. Then, $f$ and $g$ are $\big(  \log ( \frac{4L }{ \rho} ) , ~ \frac{\rho}{2}  \| \btheta^*_f - \btheta^*_g \|_2^2 + \frac{\rho}{4L^2} \| \nabla f(\btheta_f^*) - \nabla g(\btheta_g^*) \|_2^2 \big)$-close.
	\end{enumerate}
\end{lemma}

\begin{proof}
This is Lemma 5.1 in \cite{HWa23}.
\end{proof}

Finally, the notion of closeness shares some similarities with the equivalence relation, including reflexivity, symmetry, and a weak form of transitivity. Its main properties are summarized in \Cref{lem-approx} below.

\begin{lemma}\label{lem-approx}
	Let $f, g, h : \Omega \to \RR$ be lower bounded. Then,
	\begin{enumerate}
		\item\label{lem-approx-self} $f$ and $f$ are $(0, 0)$-close.
		\item\label{lem-approx-monotonicity} If $f$ and $g$ are $(\varepsilon, \delta)$-close, then $f$ and $g$ are $(\varepsilon', \delta')$-close for any $\varepsilon' \geq \varepsilon$ and $\delta' \geq \delta$.
		\item\label{lem-approx-shift} If $f$ and $g$ are $(\varepsilon, \delta)$-close and $a, b \in \RR$, $f + a$ and $g + b$ are $(\varepsilon , \delta )$-close.
		\item\label{lem-approx-symmetry} If $f$ and $g$ are $(\varepsilon, \delta)$-close, then $g$ and $f$ are $( \varepsilon,  \delta )$-close.
		\item\label{lem-approx-transitivity} If $f$ and $g$ are $(\varepsilon_1, \delta_1)$-close, and $g$ and $h$ are $(\varepsilon_2, \delta_2)$-close, then $f$ and $h$ are $ (  \varepsilon_1 + \varepsilon_2 ,  \delta_1  + \delta_2  )$-close.
		\item\label{lem-approx-general} If $\sup_{ \btheta \in \Omega } f ( \btheta ) - \inf_{ \btheta \in \Omega } f ( \btheta ) < F < \infty$ and $ \sup_{ \btheta \in \Omega } g ( \btheta ) - \inf_{ \btheta \in \Omega } g ( \btheta ) < G < \infty$, then $f$ and $g$ are $( 0,\max\{F,G\} )$-close. 
		\item\label{lem-approx-average} Suppose that $\{ f_i \}_{i=1}^m : \Omega \to \RR$ are lower bounded and $( \varepsilon, \delta )$-close to $g$. If $\{ \lambda_i \}_{i=1}^m \subseteq [ 0 , 1 ]$ and $\sum_{i=1}^{m} \lambda_i = 1$, then $\sum_{i=1}^{m} \lambda_i f_i$ and $g$ are $( \varepsilon , (e^{\varepsilon} + 1) \delta )$-close.
	\end{enumerate}
\end{lemma}

\begin{proof}
This is Lemma 5.2 in \cite{HWa23}.
\end{proof}

\section{Applications to Statistical Learning and Optimization}\label{sec-applications}

In this section, we present two applications of the similarity measure to statistical learning and optimization, namely, empirical risk minimization and non-stationary online optimization. We will assume that $\Omega$ is a convex and closed subset of $\RR^d$ with $\diam(\Omega)=M<\infty$.

\subsection{Application to Empirical Risk Minimization}\label{sec-ERM}

Let $\cZ$ be a sample space, $\cP$ a probability distribution over $\cZ$, and $\ell:\Omega\times\cZ\to\RR$ a known loss function. A common objective of many statistical learning problems is to minimize the population loss $F(\btheta) = \EE_{\bz\sim\cP} \left[\ell(\btheta,\bz)\right]$. As $\cP$ is generally unknown in practice, a standard approach is to collect i.i.d.~samples $\bz_1,...,\bz_n\sim\cP$ and perform \emph{empirical risk minimization}:
\[
\widehat{\btheta}_n \in \argmin_{\btheta\in\Omega} f_n(\btheta),
\quad\text{where}\quad
f_n(\btheta) =  \frac{1}{n} \sum_{i=1}^n \ell(\btheta,\bz_i).
\]
The convergence rate of the excess risk $F(\widehat{\btheta}_n)-\inf_{\btheta\in\Omega} F(\btheta)$ to $0$ is often obtained by studying the concentration of the empirical loss $f_n$ around the population loss $F$ \citep{BBL03, BBM05, Wai19}. In particular, classical concentration results establish a rate of $\cO(\sqrt{d/n})$ in the general case, and a fast rate of $\cO(d/n)$ when $F$ is strongly convex. We now show how the language of $(\varepsilon,\delta)$-closeness provides a unifying framework for deriving and describing these results.

In the general case, under light-tailed assumptions, a standard uniform concentration argument (e.g., Lemma C.6 in \cite{HWa23}) shows that with high probability,
\[
\sup_{\btheta\in\Omega} \big| \left[ f_n(\btheta) - F(\btheta) \right] - \left[ f_n(\btheta_0) - F(\btheta_0) \right] \big| \lesssim \sqrt{\frac{d}{n}},
\]
where $\btheta_0\in\Omega$. By Part \ref{lem-sufficient-sup} of \Cref{lem-sufficient}, $f_n$ and $F$ are $(0,C\sqrt{d/n})$-close for some $C>0$. As a consequence, 
\[
F(\widehat{\btheta}_n) - \inf_{\btheta\in\Omega} F(\btheta) \lesssim \sqrt{\frac{d}{n}},
\]
recovering the $\cO(\sqrt{d/n})$ rate.

In the more regular case where $F$ is strongly convex, under light-tailed assumptions, a standard uniform concentration argument (e.g., Lemma C.2 in \cite{HWa23}) shows that with high probability,
\[
\sup_{\btheta\in\Omega} \| \nabla f_n(\btheta) - \nabla F(\btheta) \|_2 \lesssim \sqrt{\frac{d}{n}}.
\]
By Part \ref{lem-sufficient-grad-square} of \Cref{lem-sufficient}, $f_n$ and $F$ are $(\log 2,Cd/n)$-close for some $C>0$. That is, for all $\btheta\in\Omega$,
\begin{align}
& F(\btheta) - \inf_{\btheta'\in\Omega} F(\btheta')
\le 
2\left( f_n(\btheta) - \inf_{\btheta'\in\Omega} f_n(\btheta') + \frac{Cd}{n} \right), \label{eqn-fast-rate-1} \\[4pt]
& f_n(\btheta) - \inf_{\btheta'\in\Omega} f_n(\btheta')
\le 
2\left( F(\btheta) - \inf_{\btheta'\in\Omega} F(\btheta') + \frac{Cd}{n} \right). \label{eqn-fast-rate-2}
\end{align}
Substituting $\btheta = \widehat{\btheta}_n$ into \eqref{eqn-fast-rate-1} yields
\[
F(\widehat{\btheta}_n) - \inf_{\btheta\in\Omega} F(\btheta) \lesssim \frac{d}{n},
\]
establishing the fast rate $\cO(d/n)$.

\begin{remark}[The role of $\varepsilon$]
In \eqref{eqn-fast-rate-1} and \eqref{eqn-fast-rate-2}, the multiplicative factor $e^{\varepsilon} = 2$ is crucial for obtaining the fast rate $\delta \asymp d/n$. If we did not allow such a multiplicative factor and set $e^{\varepsilon}=1$, then the similarity measure would reduce to the sup-norm metric $\delta = \| f_n-F\|_{\infty}$, which in general has rate no faster than $\sqrt{d/n}$. Similar ideas have been used in the peeling technique \citep{vdG00} for deriving local Rademacher complexity bounds \citep{BBM05,Wai19}. In \Cref{sec-ERM-general}, we illustrate how these results can be rephrased in the language of $(\varepsilon,\delta)$-closeness.
\end{remark}

\subsection{Application to Non-Stationary Online Optimization}

In non-stationary online optimization, a learner sequentially makes decisions to minimize a loss function that is changing over time \citep{Zin03, HSe09, BGZ15}. Specifically, the environment is represented by a sequence of loss functions $\{F_t\}_{t=1}^{T}:\Omega\to\RR$. At each time $t=1,2,...,T$, a learner makes a decision $\btheta_t\in\Omega$, and incurs a loss $F_t(\btheta_t)$. The goal of the learner is to minimize the cumulative loss. As an example, the functions $\{F_t\}_{t=1}^T$ can take the form $F_t(\btheta) = \EE_{\bz\sim\cP_t}\left[ \ell(\btheta,\bz) \right]$, where $\{\cP_t\}_{t=1}^T$ represents a data distribution changing over time.

The difficulty of this problem is reflected by the amount of variation in the sequence $\{F_t\}_{t=1}^T$. If the environment is near-stationary, say $F_1\approx\cdots \approx F_T$, then at each time $t$, the decisions in the past $\{\btheta_i\}_{i=1}^t$ and their losses $\{F_i(\btheta_i)\}_{i=1}^t$ provide valuable information for minimizing the upcoming loss function $F_t$. At the other extreme, if the functions $F_1,...,F_T$ are entirely different, then the historical data is no longer relevant to the future optimization tasks.

We now show that the notion of $(\varepsilon,\delta)$-closeness leads to several existing measures of variation. For simplicity, we focus on the function pair $F_{t-1}$ and $F_t$.
\begin{itemize}
\item In the general case, Part \ref{lem-sufficient-sup} of \Cref{lem-sufficient} implies that $F_{t-1}$ and $F_t$ are $(0,2\|F_{t-1}-F_t\|_{\infty})$-close, which gives the functional variation metric $\|F_{t-1}-F_t\|_{\infty}$ \citep{BGZ15, JRS15, Wan23}.
\item Alternatively, Part \ref{lem-sufficient-grad-sup} of \Cref{lem-sufficient} implies that $F_{t-1}$ and $F_t$ are $(0,2M\sup_{\btheta\in\Omega}\|\nabla F_{t-1}(\btheta)- \nabla F_t(\btheta)\|_2)$-close, which yields the gradient variation metric $\sup_{\btheta\in\Omega}\|\nabla F_{t-1}(\btheta)- \nabla F_t(\btheta)\|_2$ \citep{CYL12, ZZZ24}.
\item If $F_{t-1}$ and $F_t$ are $\rho$-strongly convex and $L$-smooth and have minimizers $\btheta_{t-1}^*$ and $\btheta_t^*$ in the interior of $\Omega$, respectively, then Part \ref{lem-sufficient-minimizers} of \Cref{lem-sufficient} implies that they are $(\log(4L/\rho),\frac{\rho}{2}\| \btheta_{t-1}^* - \btheta_t^* \|_2^2)$-close. In this case, the variation is measured by the distance between minimizers $\| \btheta_{t-1}^* - \btheta_t^* \|_2$, which is also a widely used variation metric \citep{Zin03, JRS15, MSJ16, ZZh21}.
\end{itemize}
The significance of this unification is that once we obtain performance bounds expressed in terms of the closeness between the functions $\{F_t\}_{t=1}^T$, we can immediately derive bounds based on the various variation metrics discussed above. See Section 4 of \cite{HWa23} for illustrations.

\section{Application to Local Rademacher Complexity Bounds}\label{sec-ERM-general}

In this section, we generalize the results in \Cref{sec-ERM} and show how classical local Rademacher complexity bounds \citep{BBM05,BBL05} can be rephrased in the language of $(\varepsilon,\delta)$-closeness. Let $\cH$ be a function class, $\cZ$ a sample space, $\cP$ a probability distribution over $\cZ$, and $\ell:\cH\times\cZ\to\RR$ a known loss function satisfying $|\ell(\cdot,\cdot)| \le b$ for some $b>0$. Consider the task of minimizing the population loss $L(h) = \EE_{\bz\sim\cP} \left[ \ell(h,\bz) \right]$. We take i.i.d.~samples $\bz_1,...,\bz_n\sim\cP$, and perform empirical risk minimization over a subclass $\cH_{\Omega}\subseteq\cH$:
\[
\widehat{h}_n \in \argmin_{h \in \cH_{\Omega}} L_n(h),\quad\text{where}\quad L_n(h) = \frac{1}{n} \sum_{i=1}^n \ell(h,\bz_i).
\]
Here $\cH_{\Omega}$ can be a parametrized class $\cH_{\Omega} = \{h_{\btheta} : \btheta\in\Omega\}$. Our goal is to establish the closeness between empirical loss $L_n$ and the population loss $L$, which reveals the convergence rate of the empirical risk minimizer $\widehat{h}_n$. For simplicity, we assume the existence of a population loss minimizer $\bar{h}\in\argmin_{h\in\cH_{\Omega}} L(h)$.

We also impose the following noise condition.
\begin{assumption}[Noise condition]\label{assumption_noise}
There exist $h^*\in\cH$, $C>0$ and $\alpha\in[0,1]$ such that
\[
L(h^*) \le L(h)
\quad\text{and}\quad
\EE_{\bz\sim\cP} \left[ \left( \ell(h,\bz) - \ell(h^*,\bz) \right)^2 \right] \le  C \big[ L(h) - L(h^*) \big]^{\alpha},
\quad \forall h\in\cH_{\Omega}.
\]
\end{assumption}
Assumption \ref{assumption_noise} is standard in statistical learning. For example, it holds for bounded least squares regression with $\alpha=1$ and $h^*$ as the Bayes least-squares estimate \citep{BBM05,Wai19}, and for binary classification under the Mammen-Tsybakov noise condition, with $h^*$ as the Bayes classifier \citep{MTs99,Tsy04,BBL05}.

Our closeness result will be stated in terms of the local Rademacher complexity. We adopt the following standard set-up \citep{BBM05,BBL05}. For a function class $\cC\subseteq\RR^{\cZ}$, define its Rademacher complexity by
\[
R_n(\cC) = \EE \left[ \sup_{g\in\cC} \frac{1}{n} \sum_{i=1}^n \sigma_i g(\bz_i) \right],
\]
where $\{\sigma_i\}_{i=1}^n$ are i.i.d.~Rademacher random variables independent of $\{\bz_i\}_{i=1}^n$. Define the loss class $\cG = \{\ell(h,\cdot)-\ell(h^*,\cdot) : h\in\cH_{\Omega}\}$ and its star hull $\cG^* = \{ \alpha g : \alpha\in[0,1],\,g\in\cG \}$. Let $\psi:\RR_+\to\RR_+$ be an increasing continuous functions such that
\[\psi(r) \ge R_n \big(\{ g \in \cG^* : \EE_{\bz\sim\cP} \left[ g(\bz)^2 \right] \le r^2 \} \big),
\]
and that $r\mapsto \psi(r)/r$ is decreasing on $(0,\infty)$. Define $w:\RR_+\to\RR_+$ by $w(r) = \sqrt{C }r^{\alpha/2}$.

The following result establishes the closeness between the losses $L_n:\cH_{\Omega}\to\RR$ and $L:\cH_{\Omega}\to\RR$.

\begin{prop}\label{prop-closeness-local}
Suppose Assumption \ref{assumption_noise} holds. Let $r^*$ be a solution to the equation $\psi(w(r)) = r$. Let $\bar{h} \in \argmin_{h\in\cH_{\Omega}} L(h)$. Choose $\gamma\in(0,1)$. Then with probability at least $1-\gamma$, the functions $L_n:\cH_{\Omega}\to\RR$ and $L:\cH_{\Omega}\to\RR$ are $(\log 2,\delta)$-close, where
\[
\delta = \left[ L(\bar{h}) - L(h^*) \right] + 16  \left[ 2r^* + \left( C(r^*)^{\alpha-1} + \frac{2b}{3} \right) \frac{\log(4/\gamma)}{n} \right].
\]
\end{prop}

Under Assumption \ref{assumption_noise} and mild regularity assumptions on the function class $\cH_{\Omega}$ (e.g., bounded VC dimension), one may take $\psi(r)\asymp r /\sqrt{n}$ and $r^* \asymp n^{-1/(2-\alpha)}$ up to logarithmic factors. In this case, $L_n$ and $L$ are $(\log 2,\delta)$-close with $\delta \lesssim L(\bar{h}) - L(h^*) + n^{-1/(2-\alpha)}$. We refer to \cite{BBM05,BBL05,Wai19} for detailed examples.

\begin{proof}[\bf Proof of \Cref{prop-closeness-local}]
We need to show that for all $h\in\cH_{\Omega}$,
\begin{align}
& L(h) - L(\bar{h}) \le 2 \big[ L_n(h) - L_n(\widehat{h}_n) + \delta \big], \label{eqn-local-closeness-goal-1} \\[4pt]
& L_n(h) - L_n(\widehat{h}_n) \le 2 \big[L(h) - L(\bar{h})  + \delta \big]. \label{eqn-local-closeness-goal-2}
\end{align}
Let $\bar{g} = \ell(\bar{h},\cdot) - \ell(h^*,\cdot)$. Define
\[
U(\gamma) = 16 \left[ 2r^* + \left( C(r^*)^{\alpha-1} + \frac{2b}{3} \right) \frac{\log(4/\gamma)}{n} \right].
\]
Take $\bz\sim\cP$. Standard localization arguments (e.g., Section 5.3.5 in \cite{BBL05}) show that with probability $1-\gamma$, for all $g\in\cG$,
\begin{align}
& \EE\left[ g(\bz) - \bar{g}(\bz) \right] - \frac{1}{n} \sum_{i=1}^n \left[ g(\bz_i) - \bar{g}(\bz_i) \right]  \le \sqrt{U(\gamma) \cdot \max\{\EE[g(\bz)],\, U(\gamma)\}}, \label{eqn-self-bound-1} \\[4pt]
& \frac{1}{n} \sum_{i=1}^n \left[ g(\bz_i) - \bar{g}(\bz_i) \right] - \EE\left[ g(\bz) - \bar{g}(\bz) \right] \le \sqrt{U(\gamma) \cdot \max\{\EE[g(\bz)],\, U(\gamma)\}}. \label{eqn-self-bound-2}
\end{align}

\paragraph{Case 1.} We first consider the case $\EE[g(\bz)] > U(\gamma)$.
Solving for $\EE[g(\bz)]$ in \eqref{eqn-self-bound-1} yields
\[
\EE[g(\bz)] \le 2 \left( \frac{1}{n} \sum_{i=1}^n \left[ g(\bz_i) - \bar{g}(\bz_i) \right] + \EE[\bar{g}(\bz)]  \right) + U(\gamma).
\]
The correspondence $g = \ell(h,\cdot) - \ell(h^*,\cdot)$ shows that for all $h\in\cH_{\Omega}$,
\begin{align}
L(h) - L(\bar{h})
&=
\EE\left[ g(\bz) - \bar{g}(\bz) \right] \notag \\
&\le 
\frac{2}{n} \sum_{i=1}^n \left[ g(\bz_i) - \bar{g}(\bz_i) \right] + \EE[\bar{g}(\bz)] + U(\gamma) \notag \\[4pt]
&=
2 \left[ L_n(h) - L_n(\bar{h}) \right] + \left[ L(\bar{h}) - L(h^*) \right] + U(\gamma) \label{eqn-local-close-0} \\[4pt]
&\le 
2 \left\{ L_n(h) - L_n(\widehat{h}_n)  + \left[ L(\bar{h}) - L(h^*) \right] + U(\gamma) \right\}, \label{eqn-local-close-1}
\end{align}
which proves \eqref{eqn-local-closeness-goal-1}.

To prove \eqref{eqn-local-closeness-goal-2}, for all $h \in \cH_{\Omega}$,
\begin{equation}\label{eqn-local-close-2}
L_n(h) - L_n(\widehat{h}_n)
=
\big[ L_n(h) - L_n(\bar{h}) \big] + \big[ L_n(\bar{h}) - L_n(\widehat{h}_n) \big].
\end{equation}
By \eqref{eqn-self-bound-2},
\begin{align}
L_n(h) - L_n(\bar{h})
&=
\frac{1}{n} \sum_{i=1}^n \left[ g(\bz_i) - \bar{g}(\bz_i) \right] \notag \\
&\le 
\EE\left[ g(\bz) - \bar{g}(\bz) \right] + \sqrt{\EE[g(\bz)]} \sqrt{U(\gamma)} \notag \\[4pt]
&\le 
\EE\left[ g(\bz) - \bar{g}(\bz) \right] + \frac{\EE[g(\bz)] + U(\gamma)}{2} \notag \\[4pt]
&=
\frac{3}{2}\EE\left[ g(\bz) - \bar{g}(\bz) \right] + \frac{1}{2} \EE [\bar{g}(\bz)] + \frac{1}{2} U(\gamma) \notag \\[4pt]
&=
\frac{3}{2} \left[ L(h) - L(\bar{h}) \right] + \frac{1}{2} \left[ L(\bar{h}) - L(h^*) \right] + \frac{1}{2} U(\gamma). \label{eqn-local-close-3}
\end{align}
Setting $h = \widehat{h}_n$ in \eqref{eqn-local-close-0} and noting that the quantity is non-negative gives
\begin{equation}\label{eqn-local-close-4}
L_n(\bar{h}) - L_n(\widehat{h}_n)
\le 
\frac{1}{2} \left[ L(\bar{h}) - L(h^*) \right] + \frac{1}{2} U(\gamma). 
\end{equation}
Substituting \eqref{eqn-local-close-3} and \eqref{eqn-local-close-4} into \eqref{eqn-local-close-2}, we obtain that for all $h \in \cH_{\Omega}$,
\begin{align}
L_n(h) - L_n(\widehat{h}_n)
&\le 
\frac{3}{2} \left[ L(h) - L(\bar{h}) \right] + \left[ L(\bar{h}) - L(h^*) \right] +  U(\gamma) \notag \\[4pt]
&\le 
2 \Big\{ L(h) - L(\bar{h})  + \left[ L(\bar{h}) - L(h^*) \right] + U(\gamma) \Big\}. \label{eqn-local-close-5}
\end{align}

\paragraph{Case 2.} Now consider the case $\EE[g(\bz)] \le U(\gamma)$. The argument is similar to that for case 1. By \eqref{eqn-self-bound-1} and the correspondence $g = \ell(h,\cdot) - \ell(h^*,\cdot)$, we have that for all $h\in\cH_{\Omega}$,
\begin{align}
L(h) - L(\bar{h})
&=
\EE\left[ g(\bz) - \bar{g}(\bz) \right] \notag \\
&\le 
\frac{1}{n} \sum_{i=1}^n \left[ g(\bz_i) - \bar{g}(\bz_i) \right] + U(\gamma) \notag \\[4pt]
&=
L_n(h) - L_n(\bar{h})  + U(\gamma) \label{eqn-local-close-0-case-2} \\[4pt]
&\le 
L_n(h) - L_n(\widehat{h}_n)  + U(\gamma), \label{eqn-local-close-1-case-2}
\end{align}
which proves \eqref{eqn-local-closeness-goal-1}. To prove \eqref{eqn-local-closeness-goal-2}, for all $h\in\cH_{\Omega}$,
\begin{equation}\label{eqn-local-close-2-case-2}
L_n(h) - L_n(\widehat{h}_n)
=
\big[ L_n(h) - L_n(\bar{h}) \big] + \big[ L_n(\bar{h}) - L_n(\widehat{h}_n) \big].
\end{equation}
By \eqref{eqn-self-bound-2},
\begin{align}
L_n(h) - L_n(\bar{h})
&=
\frac{1}{n} \sum_{i=1}^n \left[ g(\bz_i) - \bar{g}(\bz_i) \right] \notag \\
&\le 
\EE\left[ g(\bz) - \bar{g}(\bz) \right] + U(\gamma)  \notag \\[4pt]
&=
L(h) - L(\bar{h}) + U(\gamma). \label{eqn-local-close-3-case-2}
\end{align}
Moreover, setting $h = \widehat{h}_n$ in \eqref{eqn-local-close-0-case-2} and noting that the quantity is non-negative gives
\begin{equation}\label{eqn-local-close-4-case-2}
L_n(\bar{h}) - L_n(\widehat{h}_n)
\le 
U(\gamma). 
\end{equation}
Substituting \eqref{eqn-local-close-3-case-2} and \eqref{eqn-local-close-4-case-2} into \eqref{eqn-local-close-2-case-2}, we obtain that for all $h\in\cH_{\Omega}$,
\begin{equation}
L_n(h) - L_n(\widehat{h}_n)
\le 
L(h) - L(\bar{h}) + 2U(\gamma). \label{eqn-local-close-5-case-2}
\end{equation}

Combining \eqref{eqn-local-close-1}, \eqref{eqn-local-close-5}, \eqref{eqn-local-close-1-case-2} and \eqref{eqn-local-close-5-case-2} finishes the proof.
\end{proof}

{
\bibliographystyle{ims}
\bibliography{bib}

\begin{thebibliography}{18}
\expandafter\ifx\csname natexlab\endcsname\relax\def\natexlab#1{#1}\fi
\expandafter\ifx\csname url\endcsname\relax
  \def\url#1{\texttt{#1}}\fi
\expandafter\ifx\csname urlprefix\endcsname\relax\def\urlprefix{URL }\fi

\bibitem[{Bartlett et~al.(2005)Bartlett, Bousquet and Mendelson}]{BBM05}
\textsc{Bartlett, P.~L.}, \textsc{Bousquet, O.} and \textsc{Mendelson, S.}
  (2005).
\newblock {Local Rademacher complexities}.
\newblock \textit{The Annals of Statistics} \textbf{33} 1497 -- 1537.
\newline\urlprefix\url{https://doi.org/10.1214/009053605000000282}

\bibitem[{Besbes et~al.(2015)Besbes, Gur and Zeevi}]{BGZ15}
\textsc{Besbes, O.}, \textsc{Gur, Y.} and \textsc{Zeevi, A.} (2015).
\newblock Non-stationary stochastic optimization.
\newblock \textit{Operations Research} \textbf{63} 1227--1244.
\newline\urlprefix\url{https://doi.org/10.1287/opre.2015.1408}

\bibitem[{Boucheron et~al.(2005)Boucheron, Bousquet and Lugosi}]{BBL05}
\textsc{Boucheron, S.}, \textsc{Bousquet, O.} and \textsc{Lugosi, G.} (2005).
\newblock Theory of classification: a survey of some recent advances.
\newblock \textit{ESAIM: Probability and Statistics} \textbf{9} 323--375.
\newline\urlprefix\url{https://doi.org/10.1051/ps:2005018}

\bibitem[{Bousquet et~al.(2004)Bousquet, Boucheron and Lugosi}]{BBL03}
\textsc{Bousquet, O.}, \textsc{Boucheron, S.} and \textsc{Lugosi, G.} (2004).
\newblock \textit{Introduction to Statistical Learning Theory}.
\newblock Springer Berlin Heidelberg, Berlin, Heidelberg, 169--207.
\newline\urlprefix\url{https://doi.org/10.1007/978-3-540-28650-9_8}

\bibitem[{Chen et~al.(2019)Chen, Wang and Wang}]{CWW19}
\textsc{Chen, X.}, \textsc{Wang, Y.} and \textsc{Wang, Y.-X.} (2019).
\newblock Technical note--nonstationary stochastic optimization under
  {$L_{p,q}$}-variation measures.
\newblock \textit{Operations Research} \textbf{67} 1752--1765.
\newline\urlprefix\url{https://doi.org/10.1287/opre.2019.1843}

\bibitem[{Chiang et~al.(2012)Chiang, Yang, Lee, Mahdavi, Lu, Jin and
  Zhu}]{CYL12}
\textsc{Chiang, C.-K.}, \textsc{Yang, T.}, \textsc{Lee, C.-J.},
  \textsc{Mahdavi, M.}, \textsc{Lu, C.-J.}, \textsc{Jin, R.} and \textsc{Zhu,
  S.} (2012).
\newblock Online optimization with gradual variations.
\newblock In \textit{Proceedings of the 25th Annual Conference on Learning
  Theory} (S.~Mannor, N.~Srebro and R.~C. Williamson, eds.), vol.~23 of
  \textit{Proceedings of Machine Learning Research}. PMLR, Edinburgh, Scotland.
\newline\urlprefix\url{https://proceedings.mlr.press/v23/chiang12.html}

\bibitem[{Hazan and Seshadhri(2009)}]{HSe09}
\textsc{Hazan, E.} and \textsc{Seshadhri, C.} (2009).
\newblock Efficient learning algorithms for changing environments.
\newblock In \textit{Proceedings of the 26th Annual International Conference on
  Machine Learning}. Association for Computing Machinery, New York, NY, USA.
\newline\urlprefix\url{https://doi.org/10.1145/1553374.1553425}

\bibitem[{Huang and Wang(2023)}]{HWa23}
\textsc{Huang, C.} and \textsc{Wang, K.} (2023).
\newblock A stability principle for learning under non-stationarity.
\newblock \textit{arXiv preprint arXiv:2310.18304} .

\bibitem[{Jadbabaie et~al.(2015)Jadbabaie, Rakhlin, Shahrampour and
  Sridharan}]{JRS15}
\textsc{Jadbabaie, A.}, \textsc{Rakhlin, A.}, \textsc{Shahrampour, S.} and
  \textsc{Sridharan, K.} (2015).
\newblock {Online Optimization : Competing with Dynamic Comparators}.
\newblock In \textit{Proceedings of the Eighteenth International Conference on
  Artificial Intelligence and Statistics} (G.~Lebanon and S.~V.~N.
  Vishwanathan, eds.), vol.~38 of \textit{Proceedings of Machine Learning
  Research}. PMLR, San Diego, California, USA.
\newline\urlprefix\url{https://proceedings.mlr.press/v38/jadbabaie15.html}

\bibitem[{Mammen and Tsybakov(1999)}]{MTs99}
\textsc{Mammen, E.} and \textsc{Tsybakov, A.~B.} (1999).
\newblock {Smooth discrimination analysis}.
\newblock \textit{The Annals of Statistics} \textbf{27} 1808 -- 1829.
\newline\urlprefix\url{https://doi.org/10.1214/aos/1017939240}

\bibitem[{Mokhtari et~al.(2016)Mokhtari, Shahrampour, Jadbabaie and
  Ribeiro}]{MSJ16}
\textsc{Mokhtari, A.}, \textsc{Shahrampour, S.}, \textsc{Jadbabaie, A.} and
  \textsc{Ribeiro, A.} (2016).
\newblock Online optimization in dynamic environments: Improved regret rates
  for strongly convex problems.
\newblock In \textit{2016 IEEE 55th Conference on Decision and Control (CDC)}.
  IEEE Press.
\newline\urlprefix\url{https://doi.org/10.1109/CDC.2016.7799379}

\bibitem[{Tsybakov(2004)}]{Tsy04}
\textsc{Tsybakov, A.~B.} (2004).
\newblock {Optimal aggregation of classifiers in statistical learning}.
\newblock \textit{The Annals of Statistics} \textbf{32} 135 -- 166.
\newline\urlprefix\url{https://doi.org/10.1214/aos/1079120131}

\bibitem[{van~de Geer(2000)}]{vdG00}
\textsc{van~de Geer, S.} (2000).
\newblock \textit{Empirical Processes in M-estimation}, vol.~6.
\newblock Cambridge university press.

\bibitem[{Wainwright(2019)}]{Wai19}
\textsc{Wainwright, M.~J.} (2019).
\newblock \textit{High-Dimensional Statistics: A Non-Asymptotic Viewpoint}.
\newblock Cambridge Series in Statistical and Probabilistic Mathematics,
  Cambridge University Press.
\newline\urlprefix\url{https://doi.org/10.1017/9781108627771}

\bibitem[{Wang(2023)}]{Wan23}
\textsc{Wang, Y.} (2023).
\newblock Technical note--on adaptivity in nonstationary stochastic
  optimization with bandit feedback.
\newblock \textit{Operations Research} \textbf{0} 0.
\newline\urlprefix\url{https://doi.org/10.1287/opre.2022.0576}

\bibitem[{Zhao and Zhang(2021)}]{ZZh21}
\textsc{Zhao, P.} and \textsc{Zhang, L.} (2021).
\newblock Improved analysis for dynamic regret of strongly convex and smooth
  functions.
\newblock In \textit{Proceedings of the 3rd Conference on Learning for Dynamics
  and Control}, vol. 144 of \textit{Proceedings of Machine Learning Research}.
  PMLR.
\newline\urlprefix\url{https://proceedings.mlr.press/v144/zhao21a.html}

\bibitem[{Zhao et~al.(2024)Zhao, Zhang, Zhang and Zhou}]{ZZZ24}
\textsc{Zhao, P.}, \textsc{Zhang, Y.-J.}, \textsc{Zhang, L.} and \textsc{Zhou,
  Z.-H.} (2024).
\newblock Adaptivity and non-stationarity: Problem-dependent dynamic regret for
  online convex optimization.
\newblock \textit{Journal of Machine Learning Research} \textbf{25} 1--52.
\newline\urlprefix\url{http://jmlr.org/papers/v25/21-0748.html}

\bibitem[{Zinkevich(2003)}]{Zin03}
\textsc{Zinkevich, M.} (2003).
\newblock Online convex programming and generalized infinitesimal gradient
  ascent.
\newblock In \textit{Proceedings of the Twentieth International Conference on
  International Conference on Machine Learning}. AAAI Press.

\end{thebibliography}
}

\end{document}